\documentclass{article}
\usepackage[english]{babel}
\usepackage{url}
\usepackage{parskip}
\usepackage{caption}
\usepackage{subcaption}
\usepackage[usenames,dvipsnames]{color}
\usepackage[utf8]{inputenc}
\usepackage{float}
\usepackage{amssymb}
\usepackage{amsthm}
\usepackage[makeroom]{cancel}
\usepackage{amsmath}
\usepackage{mathtools}
\usepackage{mathrsfs}
\usepackage{amsfonts}
\usepackage{placeins}
\usepackage{graphicx}
\usepackage{geometry}
\usepackage{comment}
\usepackage{wrapfig}
\usepackage{comment}
\usepackage[useregional]{datetime2}
\usepackage{enumitem}
\usepackage{algorithm}
\usepackage{algorithmic}

\newcommand{\R}{\mathbb{R}}
\newcommand{\N}{\mathcal{N}}
\newcommand{\norm}[1]{\left\lVert#1\right\rVert}
\newtheorem{theorem}{Theorem}[section]

\theoremstyle{definition}

\theoremstyle{remark}

\allowdisplaybreaks

\title{A High-order Tuner for Accelerated Learning and Control\thanks{This work is supported by the Boeing Strategic University Initiative.}}
\author{Spencer McDonald \thanks{S. McDonald is with the Department of Aeronautics and Astronautics, Massachusetts Institute of Technology, Cambridge, MA, 02139 USA} \and Yingnan Cui \thanks{Y. Cui and A.M. Annaswamy are with the Department of Mechanical Engineering, Massachusetts Institute of Technology, Cambridge, MA, 02139 USA} \and Joseph E. Gaudio \thanks{J.E. Gaudio is with Aurora Flight Sciences, a Boeing Company, Cambridge MA 02140} \and Anuradha M. Annaswamy\footnotemark[2]}
\begin{document}
\maketitle
\begin{abstract}
	Gradient-descent based iterative algorithms pervade a variety of problems in estimation, prediction, learning, control, and optimization. Recently iterative algorithms based on higher-order information have been explored in an attempt to lead to accelerated learning. In this paper, we explore a specific a high-order tuner that has been shown to result in stability with time-varying regressors in linearly parametrized systems, and accelerated convergence with constant regressors. We show that this tuner continues to provide bounded parameter estimates even if the gradients are corrupted by noise. Additionally, we also show that the parameter estimates converge exponentially to a compact set whose size is dependent on noise statistics. As the HT algorithms can be applied to a wide range of problems in estimation, filtering, control, and machine learning, the result obtained in this paper represents an important extension to the topic of real-time and fast decision making.
\end{abstract}
\section{Introduction}
Parametric models are utilized in many fields of engineering to aid in analysis and make problems tractable; some examples include linear regression, support vector machines, neural networks, and many other techniques in machine learning \cite{Bishop_2006,Efron2016}. To determine the parameters, a loss function is chosen to measure the discrepancy between the model's predicted output and the measured output \cite{Wang2020}. To this end, gradient methods have become the central technique employed to iteratively update parameters to minimize the loss function. The investigation of these iterative algorithms has been a key research area in: the study of adaptive learning rates \cite{Duchi_2011,Kingma_2017,Wilson_2017}, time-scheduled learning rates \cite{Shalev_Shwartz_2011,Hazan_2016}, and higher order ``momentum'' based techniques \cite{Polyak_1964,Nesterov_1983,Wibisono_2016}. In particular, the iterative algorithm proposed by Nesterov \cite{Nesterov_1983} has received significant interest in both the optimization  \cite{Nesterov_2004,Beck_2009,Bubeck_2015,Carmon_2018,Nesterov_2018} and neural network communities \cite{Krizhevsky_2012,Sutskever_2013}  due to its provable guarantees of accelerated learning when restricted to the class of convex functions. 

Of interest to this current paper is the ``Higher Order Tuner'' developed in \cite{gaudio2020accelerated}. The general concept of iterative algorithms based on information other than the first-order gradients has been explored in several papers including \cite{Nesterov_1983}, with notables ones being \cite{Wibisono_2016,Su_2016,Wilson_2016,gaudio2020class}. These algorithms have been inspired by Nesterov's algorithm and its continuous-time equivalent as described in \cite{Su_2016}, with a demonstrated advantage of faster convergence in discrete time \cite{Wibisono_2016,Wilson_2016}. In the context of control problems that require real-time decision making, in addition to ensuring convergence, fast or slow, stability guarantees have to be ensured in the presence of real-time variations in the system. As dynamics can cause the underlying regressors in the problem to be time-varying and adversarial, high-order tuners that can not only lead to accelerated convergence but also are guaranteed to be stable are essential. Time varying regressors occur in many different scenarios, including multi-armed bandits \cite{Auer_1995,Auer_2002,Bubeck_2012}, adaptive filtering \cite{Goodwin_1984,Widrow_1985,Haykin_2014}, and temporal-prediction tasks \cite{Dietterich_2002,Kuznetsov_2015,Hall_2015}. The HT algorithm developed in \cite{gaudio2020accelerated} provides such a guarantee.


One of the practical considerations that was neglected is noise, which can corrupt the determination of gradients. The stochastic stability of discrete algorithms including Nesterov and Heavy Ball has been an area of active investigation in recent years when these algorithms are subjected to corrupted gradients \cite{Cohen18a,Lan2012,Devolder2013}. Gradient corruption can occur in many different ways, some examples include the case where the gradient is estimated from a subset of its components \cite{Atchade2014,Krichene2015,Jain2018}, the setting where noise is intentionally added to protect the privacy of the data \cite{Bassily2014}, and the case where noise occurs due to physical measurements \cite{Birand2013}. This paper extends on the work done in \cite{gaudio2020accelerated} to show stability and convergence rates in the presence of stochastic disturbances.

The tool set used in this work lies in the field of stochastic systems theory \cite{Kumar1986StochasticSE,Becker1985,Borkar1979}. Most deterministic discrete stability analysis is carried out using Lyapunov methods by showing that the Lypapunov function candidate decreases with every iteration \cite{Goodwin_1984} and invoking suitable convergence properties. The stochastic counterpart of these works applies the same tools to the expectation of the underlying stochastic processes and invoking subsequent real-analysis based arguments \cite{Kushner1967,Kushner1965, Kushner1971,Beutler1973}. The foundation of these stochastic Lyapunov analysis proofs relied on properties of Doob's martingale convergence theorem and LaSalle's arguments. One limiting assumption of these early works, however, required that the underlying systems be Markovian \cite{Kushner1967,Kushner1965,Kushner1971}. However, recent work in \cite{QIN2020} has expanded the stochastic Lyapunov analysis to include systems that are not Markovian and additionally show more general assumptions with regards to decreases in expectation. All of the existing work in discrete-time systems with time-varying regressors and noise including \cite{Cohen18a,Lan2012,Devolder2013}, however, have focused their attention only on first-order methods. In this paper, we propose the use of higher-order methods that deploy not only gradients but also the Hessian. 


The primary contribution of the paper is to show that the HT proposed in \cite{gaudio2020accelerated} is indeed stable even in the case when noise corrupts the gradient used to iteratively update the parameters. As in \cite{gaudio2020accelerated}, we focus our attention on regression models that are linear in the unknown parameters, and assume that the observations are corrupted by an additive noise component. By leveraging theoretical tools similar to \cite{QIN2020}, we guarantee that the parameter estimates are bounded. Additionally, when the underlying loss function is convex, we also show that the parameter estimates converge exponentially to a compact set whose size is dependent on noise statistics. As the HT algorithms can be applied to a wide range of problems in estimation, filtering, control, and machine learning, the result obtained in this paper represents an important step in real-time and fast decision making.

The remainder of this paper is organized as follows, in Section \ref{sec:problem}, we introduce the mathematical notation, problem statement, and the Higher Order Tuner algorithm; in Section \ref{sec:HT} we introduce the main result of the paper and the theoretical significance; in Section \ref{sec:stability} we discuss the stability and convergence proofs of the algorithms; in Section \ref{sec:conclusion}, concluding remarks are given. 

\section{Problem Statement and Preliminaries}
\label{sec:problem}
Let \{$\eta_k$\} be a stochastic process on probability space ($\Omega$, $\mathcal{F}$, $\mathbb{P}$) where $\eta_k$ satisfies:
\begin{gather} 
	E(\eta_{k+1} | \mathcal{F}_{k}) = d_{k+1} \,\, a.s. \label{eq:mean}\\
	E(\eta_{k+1}^2 | \mathcal{F}_{k}) = \sigma_{k+1}^2 \,\, a.s. \label{eq:variance}
\end{gather}
where $\mathcal{F}_{k}$ is a filtration on the sub-sigma increasing algebra group of $\sigma_k(\theta_1,\vartheta_1, ..., \theta_k,\vartheta_k)$ (i.e., the filtration $\mathcal{F}_{k}$ assumes we know all values of $\theta_i$ and $\vartheta_i$, $\forall i \leq k$)\cite{QIN2020}. We also assume that
\begin{gather} 
	|d_k | \leq d_{max} < \infty, \,\, \forall k \,\, a.s. \label{eq:meanB}\\
	\sigma_k^2 \leq \sigma_{max}^2 < \infty, \,\, \forall k \,\, a.s. \label{eq:varianceB}
\end{gather}

With the noise characteristics defined as above, we now state the problem that will address in this paper: Suppose that there is a vector of measurements $\phi_k\in\R^N$ that produces a noisy output $y_k\in\R$ defined by the following linear regression model
\begin{equation}\label{eq:noisy_model_output}
	y_{k+1}=\phi_k^T\theta^*+\eta_{k+1}.
\end{equation}
where $\theta^*\in\R^N$ represents a vector of constant unknown parameters and needs to be estimated. A recursive estimator of the form 
\begin{equation}\label{eq:estimate}
	\hat{y}_{k+1}=\theta^T_k\phi_k.
\end{equation}
is proposed for the estimation, with $\theta_k\in\R^N$ and $\hat{y}_k\in\R$ as the estimates of the parameter, and output respectively. The regressor model and the estimator lead to an output error equation 
\begin{equation}
	\label{eq:noisefree_errormodel}
	e_{y,k+1} = \tilde{\theta}^T_k\phi_k-\eta_{k+1}
\end{equation}
where $\tilde{\theta}_k\in\R^n$ is defined as $\tilde{\theta}_k = \theta_k - \theta^*$ and $e_{y,k}= \hat{y}_{k}-y_{k}$. This leads to a loss function, $L_{k+1}$ 
\begin{equation}\label{eq:noisefree_loss}
	L_{k+1}(\theta_k) = \frac{1}{2}\lVert\tilde{\theta}^T_k\phi_k\rVert^2. 
\end{equation}
As the presence of noise is unknown to the designer, the gradient of the loss function can be estimated only as
\begin{equation} \label{eq:noisefree_gradient}
	\nabla L_{k+1}(\theta_k) \approx \phi_k e_{y,k+1},
\end{equation}
leading to the well known gradient-based recursive estimation given by
\begin{equation}
	\theta_{k+1}=\theta_k-\gamma \nabla L_{k+1}(\theta_k).
	\label{eq:gradient}
\end{equation}
The stability and convergence properties of the underlying Markovian process given by \eqref{eq:noisy_model_output}-\eqref{eq:noisefree_gradient} have been studied at length in \cite{QIN2020,AHMED1972276,Kushner2014}, and is summarized in the two theorems below. The reader is referred to \cite{QIN2020} for all definitions of stability, asymptotic stability in probability, and for the definition of almost surely exponentially stable in the large. 
\begin{theorem} \label{thm:stochasticstable}
	Let $\{X_k\}_{k\geq1}$ be a stochastic process and let $V: \mathbb{R}^n \to \mathbb{R}$ be a continuous, nonnegative, and radially unbounded function. Define the set $D_{\lambda} \coloneqq  \{x: \,\, V(x) < \lambda\}$ for some $\lambda > 0$, and assume that for any k, $E(V(X_{k+1} )| \mathcal{F}_k) - V(X_k) \leq -\psi(X_k)$, where $\psi: \mathbb{R}^n \to \mathbb{R}$ is continuous and satisfies $\psi(x) \geq 0$ for any $x \in D_{\lambda}$.
	The following results can be stated:
	\begin{enumerate}
		\item For any initial condition $X_0 \in D_{\lambda}$, $X_k$ converges to $D_0 \coloneqq \{x \in D_{\lambda} : \psi(x) = 0\}$ with probability at least $1-\frac{V(X_0)}{\lambda}$;
		\item If moreover $\psi(x)$ is positive definite on $D_{\lambda}$, and $h_1(\norm{s}) \leq V (s) \leq h_2 (\norm{s}) $ for two class $\kappa$ functions $h_1$ and $h_2$, then $x = 0$ is asymptotically stable in probability.
	\end{enumerate}
\end{theorem}
The reader is referred to \cite{QIN2020} for a proof, and to \cite{AHMED1972276,Kushner2014} for the fundamentals of Markovian processes and properties of convergence. The following theorem derives stronger stability properties of \eqref{eq:noisy_model_output}-\eqref{eq:noisefree_gradient} under stronger conditions on the function $V$.
\begin{theorem} \label{thm:stochasticrate}
	Let $\{X_k\}_{k\geq1}$ be a stochastic process and let $V: \mathbb{R}^n \to \mathbb{R}$ be a continuous, nonnegative, and radially unbounded function. Assume that we have $E(V(X_{k+1} )| \mathcal{F}_k) - V(X_k) \leq -\alpha V(X_k)$, where $0 < \alpha < 1$.
	The following results can be stated:
	\begin{enumerate}
		\item For any initial condition $X_0$, we have that $V(X_k)$ converges to 0 exponentially fast at a rate no slower than $1-\alpha$. (\cite{Kushner1971}, Thm. 2, Chapter 8]) and \cite{Kushner2014};
		\item If moreover $V$ satisfies $c_1 \norm{x}^a \leq V (x) \leq c_2 \norm{x}^a$ for some $c_1,c_2,a > 0$, then $X=0$ is globally almost surely exponentially stable.
	\end{enumerate}
\end{theorem}
The reader is referred to \cite{QIN2020} for a proof, and to \cite{AHMED1972276,Kushner2014} for the fundamentals of Markovian processes and properties of convergence.

\section{The High-order Tuner Algorithm}
\label{sec:HT}
We now state the main result of this paper. We consider the problem of estimation of $\theta^*$ in \eqref{eq:noisy_model_output} using a recursive estimator in \eqref{eq:estimate}. Instead of the gradient algorithm in \eqref{eq:gradient}, we propose one based on a high-order tuner, summarized in Algorithm \ref{alg:HOT_1}. In order to accommodate the presence of noise, we consider a regularized loss function of the form\cite{gaudio2020accelerated}
\begin{equation}
	\label{eq:objective}
	f_{k+1}(\theta_k) = \frac{L_{k+1}(\theta_k)}{\N_k} + \frac{\mu}{2}\|\theta_k - \theta_0\|^2,
\end{equation}
which leads to a strongly convex objective function.
\begin{algorithm}[H]
	\caption{High Order Tuner\cite{gaudio2020accelerated}}
	\label{alg:HOT_1}
	\begin{algorithmic}[1]
		\STATE {\bfseries Input:} initial conditions $\theta_0$, $\vartheta_0$, gains $\gamma$, $\beta$, $\mu$
		\FOR{$k=0,1,2,\ldots$}
		\STATE \textbf{Receive} regressor $\phi_k$, output $y_{k+1}$
		\STATE Let $\N_k=1+\lVert\phi_k\rVert^2$, $\nabla L_{k+1}(\theta_k)=\phi_k(\theta_k^T\phi_k-y_{k+1})$,\\
		$\nabla f_{k+1}(\theta_k)=\frac{\nabla L_{k+1}(\theta_k)}{\N_k}+\mu(\theta_k-\theta_0)$,\\ $\bar{\theta}_{k}=\theta_k-\gamma\beta\nabla f_{k+1}(\theta_k)$
		\STATE $\theta_{k+1}\leftarrow\bar{\theta}_{k}-\beta(\bar{\theta}_{k}-\vartheta_k)$
		\STATE Let $\nabla L_{k+1}(\theta_{k+1})=\phi_k(\theta_{k+1}^T\phi_k-y_{k+1})$,\\
		$\nabla f_{k+1}(\theta_{k+1})=\frac{\nabla L_{k+1}(\theta_{k+1})}{\N_k}+\mu(\theta_{k+1}-\theta_0)$
		\STATE $\vartheta_{k+1}\leftarrow\vartheta_k-\gamma \nabla f_{k+1}(\theta_{k+1})$
		\ENDFOR
	\end{algorithmic}
\end{algorithm}

The specific updates that define the evolution of $\theta_k$ are summarized as
\begin{align}
	\vartheta_{k+1} &= \vartheta_k - \gamma \nabla f_{k+1}(\theta_{k+1}),\label{eq:4}\\
	\bar{\theta}_{k} &= \theta_k - \gamma\beta\nabla f_{k+1}(\theta_k),\label{eq:5}\\
	\theta_{k+1} &= \bar{\theta}_{k} - \beta(\bar{\theta}_{k} - \vartheta_k),\label{eq:6}
\end{align}
where $\nabla f_{k+1}(\theta_k) = \frac{\nabla L_{k+1}(\theta_k)}{\N_k} + \mu(\theta_k - \theta_0)$. It should be noted that \eqref{eq:4}-\eqref{eq:5} represent a high order tuner in two steps, and $\N_k$ includes the Hessian of the loss function. A similar approach is commonly utilized in adaptive control, leading to the well known $\sigma$-modification \cite{Ioannou1996}. Since the gradient of the loss function can be determined only as in \eqref{eq:noisefree_gradient}, the two gradients in \eqref{eq:4} and \eqref{eq:5} take the form
\begin{align}
	\nabla L_{k+1}(\theta_k)&=\phi_k(\theta_k^T\phi_k- {\theta^*}^T\phi_k-\eta_{k+1}), \label{eq:modified_HOT_algorithm_1}\\
	\nabla L_{k+1}(\theta_{k+1})&=\phi_k(\theta_{k+1}^T\phi_k- {\theta^*}^T\phi_k-\eta_{k+1}). \label{eq:modified_HOT_algorithm_2}
\end{align}
In the following two sections, we analyze the stability and convergence properties of \eqref{eq:4}-\eqref{eq:modified_HOT_algorithm_2}.

\section{Stability and convergence}
\label{sec:stability}
In this section, we will show the stability of the high-order tuner algorithm in the presence of noise. First we show the boundedness of the parameters $\vartheta_k$ and $\theta_k$ in Section \ref{sec:bound}. Then based on the analysis for the parameter boundedness, we show that the exponential convergence rate in Section \ref{sec:converge}.
\subsection{Stability of the High-Order Tuner Algorithm}
\label{sec:bound}

\begin{theorem}
	For the stochastic linear regression model with output error $e_{y,k+1}$ as defined in \eqref{eq:noisefree_errormodel}, and assuming \eqref{eq:mean}, \eqref{eq:variance}, \eqref{eq:meanB} and \eqref{eq:varianceB} are satisfied, Algorithm \ref{alg:HOT_1} in \eqref{eq:4}-\eqref{eq:modified_HOT_algorithm_2} with $0<\mu<1$, $0<\beta<1$, $0<\gamma\leq \frac{\beta(2-\beta)}{16+\beta^2+\mu\left(\frac{57\beta+1}{16\beta}\right)}$ results in $\vartheta-\theta^*\in\ell_{\infty}$, $\theta-\vartheta\in\ell_{\infty}$, with probability one.
	\label{theo:3}
\end{theorem}
\begin{proof}
	Consider the candidate Lyapunov function
	\begin{equation}
		\label{eq:lyap}
		V(x_k) = V_k=\frac{1}{\gamma}\lVert \vartheta_k-\theta^*\rVert^2+\frac{1}{\gamma}\lVert \theta_k-\vartheta_k\rVert^2,
	\end{equation}
	where $x_k = [\theta_k^\top, \vartheta_k^\top]^\top$. We can expand the increment $\Delta V_k:=V_{k+1}-V_k$ and take expectation using \eqref{eq:noisefree_errormodel}, \eqref{eq:noisefree_loss}, and Algorithm \ref{alg:HOT_1} to obtain (see the Appendix for details)
	\begin{align}
		E&(V_{k+1} | \mathcal{F}_k) - V_k\nonumber\\
		\leq &-\underbrace{\frac{10}{16}\mu\gamma\beta}_{c_1}V_k + \underbrace{\frac{19609}{6144} d_{max}}_{c_2} \sqrt{V_k}\nonumber \\ 
		&+ \underbrace{\mu\left(\frac{3570\beta+896}{224\beta}\right)\lVert\theta^*-\theta_0\rVert^2}_{c_3}\nonumber \\
		&+\underbrace{\frac{67}{256}d_{max} \norm{ \theta_0} +\frac{1}{8} d_{max} \norm{(\theta^* - \theta_0)}+ \frac{15001}{1536} d_{max} \norm{\theta^{*}}}_{c_4}\nonumber\\
		&+ \underbrace{4 \gamma \beta \sigma_{max}^2 \left|1 - \frac{3}{2}\beta\right|+ 2 (1 - \beta) \gamma \beta \sigma_{max}^2 +2\gamma \sigma_{max}^2}_{c_5} .\label{eq:7}
	\end{align}
	Defining the compact set $D$ as
	\begin{equation*}
		D = \left\{V\middle |\, V\leq K\right\}
	\end{equation*}
	where $\hat{c} = c_3 + c_4 + c_5$ and $K$ is the greatest positive real root of $-c_1 x + c_2\sqrt{x} + \hat{c} = 0$, the inequality in \eqref{eq:7} can be restated as
	\begin{equation}
		E(V_{k+1} | \mathcal{F}_k) - V_k < 0,
		\label{eq:8}
	\end{equation}
	in $D^c$. In order to show the properties of $V_k$ outside of the compact set $D$, we consider a new function $\hat{V}_k$, defined as
	\begin{equation}
		\label{eq:10}
		\hat{V}(x_k) = \hat{V}_k = \left\{
		\begin{array}{ll}
			V_k - K & \quad V_k > K\\
			0 & \quad V_k \leq K
		\end{array}
		\right.
	\end{equation}
	It is obvious that $\hat{V}$ is a nonnegative, continuous, and radially unbounded function. Now we show that
	\begin{equation}
		E(\hat{V}_{k+1} | \mathcal{F}_k) - \hat{V}_k \leq -\psi(x_k), \,\, \forall x_k \in \mathbb{R}^n,
		\label{eq:1}
	\end{equation}
	where $\psi$ is a continuous positive semi-definite function. 
	
	Let $\psi$ be defined as the following:
	\begin{equation}
		\psi(x_k) = \left\{
		\begin{array}{ll}
			c_1 V(x_k) - c_2 \sqrt{V(x_k)} - \hat{c} - K& \quad V(x_k)> T\\
			0 & \quad V(x_k) \leq T
		\end{array}
		\right.
		\label{eq:9}
	\end{equation}
	and $T$ as
	\begin{equation*}
		T =  \frac{c_2^2 + 2 c_1 (\hat{c}+K) + \sqrt{c_2^4 + 4c_1 c_2^2 (\hat{c}+K)}}{2c_1^2},
	\end{equation*}
	such that $T$ is the largest positive root of  $c_1 x -c_2 \sqrt{x} -\hat{c} - K = 0$. Since $V$ is a continuous function, it is obvious that $\psi$ is a continuous positive semi-definite function. We consider three separate cases to conclude the proof:
	\begin{enumerate}
		\item $V_k \leq K$: This means $\hat{V_k} = 0$ and $-\psi(x_k) =0$. From \eqref{eq:8}, it follows that $E(V_{k+1} | \mathcal{F}_k) \leq (1-c_1) V_k + c_2 \sqrt{V_k} + \hat{c}$. Since $1-c_1>0$, we can use the upper bound $K$ to conclude that $E(V_{k+1} | \mathcal{F}_k) \leq (1-c_1) K + c_2 \sqrt{K} + \hat{c} = K$. Since the $E(V_{k+1} | \mathcal{F}_k) \leq K$, we know $E(\hat{V}_{k+1} | \mathcal{F}_k) = 0$. Finally we arrive at $E(\hat{V}_{k+1} | \mathcal{F}_k) - \hat{V}_k = 0$. From the definition in \eqref{eq:9}, \eqref{eq:1} is proved.
		\item $K < V_k \leq T$: This implies that $\hat{V}_k >0$ and $-\psi(x_k) = 0$. We know from before that $E(V_{k+1} | \mathcal{F}_k) - V_k \leq -c_1 V_k + c_2 \sqrt{V_k} + \hat{c}$. Since $V_k > K$, we know that $-c_1 V_k + c_2 \sqrt{V_k} + \hat{c} < 0$. Thus $E(V_{k+1} | \mathcal{F}_k) - V_k <0$. Algebraic manipulations lead us to the inequality $E(V_{k+1} - K | \mathcal{F}_k) - \hat{V}_k <0$. Also, from the definition of $\hat{V}_k$, we can conclude that $E(\hat{V}_{k+1} | \mathcal{F}_k) = \max\{E(V_{k+1} - K | \mathcal{F}_k), 0\}$. Since in this case $\hat{V}_k > 0$ and $\hat{V}_k >E(V_{k+1} - K | \mathcal{F}_k)$, it follows that $E(\hat{V}_{k+1} | \mathcal{F}_k) - \hat{V}_k < 0$. From \eqref{eq:9}, it follows that \eqref{eq:1} holds.
		\item $T < V_k$: This means $\hat{V}_k >0$ and $-\psi(x_k) = -c_1 V(x_k) + c_2 \sqrt{V_k} + \hat{c} +K$. We know from before that $E(V_{k+1} | \mathcal{F}_k) - V_k \leq -c_1 V(x_k) + c_2 \sqrt{V_k} + \hat{c}$. We can then conclude, $E(V_{k+1} | \mathcal{F}_k) - \hat{V_k} \leq -c_1 V(x_k) + c_2 \sqrt{V_k} + \hat{c} + K < 0$, because $V_k > T > K$. Since $\hat{V}_k \leq V_k \;\; \forall k$, we finally arrive at $E(\hat{V}_{k+1} | \mathcal{F}_k) - \hat{V_k} \leq -c_1 V(x_k) + c_2 \sqrt{V_k} + \hat{c} + K$, proving \eqref{eq:1} as well.
	\end{enumerate}
	
	If we define $E_\lambda=\{x | V(x)<\lambda\}$, it follows that \eqref{eq:1} holds for all $E_\lambda$. Since $V$ is radially unbounded, Theorem \ref{thm:stochasticstable} and \ref{thm:stochasticrate} can be applied to conclude that $x_k$ converges to the set $\{x | \psi(x_k)=0\}$ with probability one. Since $\psi(x_k) = 0$ means that $V(x_k) \leq T$, we therefore obtain $\vartheta - \theta^* \in \ell_\infty$ and $\theta - \vartheta \in \ell_\infty$ with probability one.
\end{proof}

\subsection{Convergence Properties of the High Order Tuner}
\label{sec:converge}
In this section, we consider the convergence properties of Algorithm \ref{alg:HOT_1}. In particular, it will be shown that the algorithm in \ref{alg:HOT_1} converges to a compact set exponentially fast.
\begin{theorem}
	For the stochastic linear regression model in \eqref{eq:noisefree_errormodel}, if noise assumptions \eqref{eq:mean}-\eqref{eq:varianceB} are satisfied, Algorithm \ref{alg:HOT_1} with $0<\mu<1$, $0<\beta<1$, $0<\gamma\leq \frac{\beta(2-\beta)}{16+\beta^2+\mu\left(\frac{57\beta+1}{16\beta}\right)}$ results in the sequence $V(x_k)$ converging exponentially fast to a compact set defined as  $M = \left\{V  \middle|V \leq K \right\}$, where $K = \max\left\{\frac{c_2^2}{(\alpha-c_1)^2}, \frac{\hat{c}}{\alpha}\right\}$, at a rate no slower than $1-\alpha$, for any $0<\alpha<c_1$.
	\label{theo:4}
\end{theorem}
\begin{proof}
	Consider the same Lyapunov function as in \eqref{eq:lyap}. Since the conditions of Theorem \ref{theo:3} are identical to Theorem \ref{theo:4}, equations \eqref{eq:7}, \eqref{eq:8} and \eqref{eq:10} still hold. From the choice of $\alpha$, it follows that whenever $V_k > \frac{c_2^2}{(\alpha-c_1)^2}$, the following holds,
	\begin{equation*}
		(\alpha-c_1) V_k + c_2 \sqrt{V_k} \leq 0
	\end{equation*}
	From this we can conclude that
	\begin{equation*}
		\begin{split}
			E(V_{k+1} | \mathcal{F}_k) - V_k \leq -\alpha V_k + \hat{c}, \quad\forall\, V_k\in S^c
		\end{split}
	\end{equation*}
	where
	\begin{equation*}
		S = \left\{V\middle|V\leq \frac{c_2^2}{(\alpha-c_1)^2}\right\}
	\end{equation*}
	
	We can additionally conclude that $E(V_{k+1} | \mathcal{F}_k) - V_k<0$ whenever $V_k$ is in $S^c \cap F^c$ where
	\begin{equation*}
		F = \left\{V\middle|V\leq \frac{\hat{c}}{\alpha} \right\}
	\end{equation*}
	Since $S^c \cap F^c = (S \cup F)^c$ by Demorgan's law, it follows that if
	\begin{equation*}
		G = \left\{V\middle|V\leq \frac{c_2^2}{(\alpha-c_1)^2} \lor V\leq \frac{\hat{c}}{\alpha} \right\}
	\end{equation*}
	and $K = \max\{\frac{c_2^2}{(\alpha-c_1)^2}, \frac{\hat{c}}{\alpha}\}$, $G$ is equivalent to 
	\begin{equation*}
		G = \left\{V\middle|V\leq K \right\}
	\end{equation*}
	Define $\hat{V}(x_k)$ as in \eqref{eq:10}. Now we show that
	\begin{equation}
		E(\hat{V}_{k+1} | \mathcal{F}_k) - \hat{V}_k \leq -\alpha \hat{V}_k,\quad \forall k,
		\label{eq:2}
	\end{equation}
	by considering two cases:
	\begin{enumerate}
		\item $\hat{V}_k > 0$:
		This means that $V_k > K$. Thus $V_k \in G^c$, which means that $E(V_{k+1} | \mathcal{F}_k) - V_k \leq -\alpha V_k + \hat{c}$. Since in this case $\hat{V}_k = V_k -K $, it follows that $E(V_{k+1} | \mathcal{F}_k) - V_k \leq -\alpha \hat{V}_k - \alpha K + \hat{c}$. Additionally, since $K \geq \frac{\hat{c}}{\alpha}$, it follows that $E(V_{k+1} | \mathcal{F}_k) - V_k \leq -\alpha \hat{V}_k$. By adding and subtracting $K$, we get $E(V_{k+1} - K| \mathcal{F}_k) - (V_k-K) \leq  -\alpha \hat{V}_k$. Since $\hat{V}_k > 0$, we have $E(V_{k+1} - K| \mathcal{F}_k) - \hat{V}_k \leq  -\alpha \hat{V}_k$. We now consider two subcases depending on $V_{k+1}$:
		\begin{enumerate}
			\item $V_{k+1} > K$: In this case, $E(\hat{V}_{k+1}| \mathcal{F}_k) = E(V_{k+1} - K| \mathcal{F}_k)$. Hence $E(\hat{V}_{k+1} | \mathcal{F}_k) - \hat{V}_k \leq  -\alpha \hat{V}_k$.
			\item $V_{k+1} \leq K$: In this case, $\hat{V}_{k+1} =0$. This implies that $E(\hat{V}_{k+1} | \mathcal{F}_k) =0$ and $E(\hat{V}_{k+1} | \mathcal{F}_k) - \hat{V}_k \leq -\alpha \hat{V}_k$ because $\alpha < 1$.
		\end{enumerate}
		
		Therefore we conclude in case 1) that $E(\hat{V}_{k+1} | \mathcal{F}_k) - \hat{V}_k \leq  -\alpha \hat{V}_k$.
		\item $\hat{V}_k = 0$: This means $V_k \leq K$ We know that $E(V_{k+1} | \mathcal{F}_k) - V_k \leq -c_1 V_k + c_2 \sqrt{V_k} + \hat{c}$. We can conclude then that $E(V_{k+1} | \mathcal{F}_k) \leq (1-c_1) V_k + c_2 \sqrt{V_k} + \hat{c}$. Since we also have that $1-c_1 > 0$, we can conclude that $E(V_{k+1} | \mathcal{F}_k) \leq (1-c_1) K + c_2 \sqrt{K} + \hat{c}$. Since $K \geq \frac{c_2^2}{(\alpha-c_1)^2}$, we can conclude that $-c_1 K + c_2 \sqrt{K} \leq -\alpha K$. Thus we get $E(V_{k+1} | \mathcal{F}_k) \leq -\alpha K + \hat{c} +K$. And as shown above, $-\alpha K \leq - \hat{c}$. Thus, $E(V_{k+1} | \mathcal{F}_k) \leq K$. This then implies that $E(\hat{V}_{k+1} | \mathcal{F}_k) = 0$. Thus once again we get $E(\hat{V}_{k+1} | \mathcal{F}_k) - \hat{V}_k \leq -\alpha \hat{V}_k$.
	\end{enumerate}
	
	This implies that $E(\hat{V}(x_{k+1} | \mathcal{F}_k) - \hat{V}(x_k) \leq -\alpha \hat{V}(x_k)$, $\forall x_k \in \mathbb{R}^n$. We therefore apply Theorem \ref{thm:stochasticrate} to guarantee that $\hat{V}(x_k)$ converges to zero exponentially at a rate no slower than $1-\alpha$. This implies that $V(x_k)$ converges to the set $G$ at the same exponential rate. 
\end{proof}

It should be noted that both Theorems \ref{theo:3} and \ref{theo:4} are established through the use of a regularized loss function as in \eqref{eq:objective}. This in turn enabled boundedness and convergence to a compact set without requiring any persistent excitation, similar to $\sigma$-modification. This implies that the proposed  HT in Algorithm \ref{alg:HOT_1}, from \cite{gaudio2020accelerated} is expected to retain all of the stability and robustness properties that have been established with modified gradient-based adaptive laws in the adaptive control literature. 


\section{Conclusions}
\label{sec:conclusion}
In this paper, we explore a high-order tuner that has been shown to result in stability with time-varying regressors in linearly parametrized systems and accelerated convergence with constant regressors\cite{gaudio2020accelerated}. We have shown that this tuner to provides bounded parameter estimates even the gradients are corrupted by noise. A quadratic loss function with a regularization term that renders the function strongly convex is utilized in order to ensure the stability and convergence properties. A straight forward extension of the result obtained to the deterministic case when bounded external disturbances are present, and time-varying parameters are present can be carried out. By leveraging recent results \cite{moreu20}, an extension to general convex functions that are non-quadratic appears to be feasible as well. The problem of accelerated convergence in the presence of time-varying regressors, bounded parameter estimation for non-convex loss functions, and parameter convergence in the presence of excitation conditions are all topics for future investigations.

\section{appendix}
\label{sec:appendix}
\begin{proof}
	Consider the candidate Lyapunov function stated as
	\begin{equation}
		V(x_k) = V_k=\frac{1}{\gamma}\lVert \vartheta_k-\theta^*\rVert^2+\frac{1}{\gamma}\lVert \theta_k-\vartheta_k\rVert^2.
	\end{equation}
	Where $x_k = [\theta_k^T, \vartheta_k^T]^T$. The increment $\Delta V_k:=V_{k+1}-V_k$ may then be expanded using \eqref{eq:noisefree_errormodel}, \eqref{eq:noisefree_loss}, and Algorithm \ref{alg:HOT_1} as
	\begingroup
	\allowdisplaybreaks

	
	\endgroup
	For notational simplicity, write $\hat{c} = c_3 + c_4 + c_5$. From the bound on the conditional expectation, it can be noted that $E(V_{k+1} | \mathcal{F}_k) - V_k<0$ in $D^c$, where the compact set $D$ is defined as
  \begin{equation*}
    D = \left\{V\middle|V\leq \frac{c_2^2 + 2 c_1 \hat{c} + \sqrt{c_2^4 + 4c_1 c_2^2 \hat{c}}}{2c_1^2}  = K\right\}
  \end{equation*}
  where $K$ is the greatest positive real root of $-c_1 x + c_2\sqrt{x} + \hat{c}$.
\end{proof}

\bibliographystyle{IEEEtran}
\bibliography{IEEEabrv,References}
\end{document}